\documentclass{article}

\usepackage{microtype}
\usepackage{graphicx}
\usepackage{subfigure}
\usepackage{booktabs} 
\usepackage{appendix} 

\usepackage[arxiv, nohyperref]{icml2020}

\usepackage[utf8]{inputenc} 
\usepackage[T1]{fontenc}    
\usepackage{url}            
\usepackage{amsfonts}       
\usepackage{nicefrac}       
\usepackage{adjustbox}
\usepackage{makecell}       
\usepackage{arydshln}       

\newcommand{\homepath}{./}

\usepackage{amsmath,amsfonts,bm}

\def\eqref#1{equation~\ref{#1}}

\def\1{\bm{1}}

\def\eps{{\epsilon}}

\def\rx{{\textnormal{x}}}

\def\rveps{{\bm{\epsilon}}}

\def\rvr{{\mathbf{r}}}

\def\rvx{{\mathbf{x}}}

\def\rvz{{\mathbf{z}}}

\def\rmX{{\mathbf{X}}}

\def\vmu{{\bm{\mu}}}

\def\ve{{\bm{e}}}

\def\vp{{\bm{p}}}

\def\vu{{\bm{u}}}
\def\vv{{\bm{v}}}

\def\vx{{\bm{x}}}

\def\vsigma{\bm{\sigma}}

\def\evsigma{{\sigma}}

\def\mA{{\bm{A}}}

\def\mC{{\bm{C}}}

\def\mM{{\bm{M}}}

\def\mP{{\bm{P}}}

\def\mR{{\bm{R}}}

\def\mU{{\bm{U}}}
\def\mV{{\bm{V}}}
\def\mW{{\bm{W}}}
\def\mX{{\bm{X}}}

\def\mSigma{{\bm{\Sigma}}}
\def\mPsi{{\bm{\Psi}}}

\DeclareMathAlphabet{\mathsfit}{\encodingdefault}{\sfdefault}{m}{sl}
\SetMathAlphabet{\mathsfit}{bold}{\encodingdefault}{\sfdefault}{bx}{n}

\newcommand{\E}{\mathbb{E}}

\newcommand{\KL}{D_{\mathrm{KL}}}

\input{\homepath Definitions.tex}

\newcommand{\figtablepath}{./FigureTable/}
\graphicspath{\figtablepath}

\usepackage{sidecap}
\usepackage{placeins}
\usepackage{caption}
\usepackage{subfigure}
\usepackage{color}
\usepackage{tikz}
\usetikzlibrary{fit,positioning}
\usetikzlibrary{arrows,automata}

\icmltitlerunning{Variational Inference for Deep Probabilistic Canonical Correlation Analysis}

\begin{document}
\twocolumn[
\icmltitle{Variational Inference for Deep Probabilistic Canonical Correlation Analysis}

\begin{icmlauthorlist}
	\icmlauthor{Mahdi Karami}{ua}
	\icmlauthor{Dale Schuurmans}{ua}
\end{icmlauthorlist}

\icmlaffiliation{ua}{Department of Computer Science, University of Alberta, Edmonton, Alberta, Canada}

\icmlcorrespondingauthor{Mahdi Karami}{karami1@ualberta.ca}

\icmlkeywords{Multi-view, probabilistic CCA, Variational inference}

\vskip 0.3in
]



\printAffiliationsAndNotice{}  

\begin{abstract}
In this paper, we  propose a deep probabilistic multi-view model that is composed of a linear multi-view layer based on probabilistic canonical correlation analysis (CCA) description in the latent space together with deep generative networks as observation models.
The network is designed to decompose the variations of all views into a shared latent representation and a set of view-specific components where the shared latent representation is intended to describe the common underlying sources of variation among the views. 
An efficient variational inference procedure is developed that approximates the posterior distributions of the latent probabilistic multi-view layer while taking into account the solution of probabilistic CCA. 
A generalization to models with arbitrary number of views is also proposed. 
The empirical studies confirm that the proposed deep generative multi-view model 
can successfully extend deep variational inference to multi-view learning 
while it efficiently integrates  the relationship between multiple views to alleviate the difficulty of learning. 
\end{abstract}


\section{Introduction} \label{sec:intro}
When observations consist of multiple views or modalities of the same underlying source of variation,
a learning algorithm should efficiently account for the complementary information to 
alleviate learning difficulty \citep{chaudhuri2009mvClusteringCCA} 
and improve accuracy.
A well-established method for two-view analysis is given by canonical
correlation analysis (CCA) \citep{hotelling1992CCA}, 
a classical subspace learning technique 
that extracts the common information between two multivariate random variables by projecting them onto a subspace. 
CCA, as a standard model for unsupervised two-view learning, has been used in a broad range of tasks such as dimensionality reduction, visualization and time series analysis \citep{xia2014robust}.

The goal of representation learning is to capture the essence of data and extract its natural features.
Such features can be categories or cluster memberships.
In multi-view data, the relationship between different views should be leveraged by the  representation learning algorithms to enhance feature extraction.  
Learning representations in real-world applications, where the data is typically high-dimensional with complex structure, poses significant challenges and necessitates flexible and expressive yet scalable models such as deep generative neural networks to be applied.

It has been shown in \citep{chaudhuri2009mvClusteringCCA} that by projecting multi-view data onto low-dimensional subspaces using CCA, cluster memberships can be recovered under a weak separation condition thus resulting in easier clustering in the subspace. 
Nevertheless, CCA exhibits poor generalization when trained on small training sets, therefore \citep{klami2007localDependentComponents, klami2013bayesianCCA} adopts a Bayesian approach to solve a probabilistic interpretation of CCA.
However, real applications involve nonlinear subspaces where more than two view are available.
Recently, deep learning  has received renewed interest 
as a standard approach for describing highly expressive models for increasingly complex datasets. 
In multi-view learning, several deep learning based approaches have been successfully  extended 
\citep{ngiam2011multimodalDL, andrew2013DCCA, wang2015deepMV, wang2017acousticVCCA, abavisani2018deep}.

\textbf{Main contributions:}
In this work, we 
present a modified formulation of probabilistic CCA that enables an extension to deep generative models.
Generalizing the approach to an arbitrary number of views is also discussed.
In this approach, linear probabilistic layers are extended to deep generative multi-view networks 
that capture the variations of the views by a \emph{shared latent representation} that describes  most of the variability (essence) of multi-view data,
and a set of \emph{view-specific factors}.
Variational inference provides a powerful tool for scaling probabilistic models to complex problems and large scale datasets \citep{rezende2014stochastic, kingma2013VAE, rezende_shakir2015_variationalInfWithNormalizingFlows}.
Hence, to design a scalable training algorithm, 
we follow variational inference principles, 
taking into account the probabilistic CCA formulation
to achieve an approximate posterior 
for the latent linear multi-view layer.
Empirical studies
confirm that the proposed deep generative multi-view model can efficiently integrate the relationship between multiple views to alleviate learning difficulty, which is the goal of multi-view learning approaches \citep{chaudhuri2009mvClusteringCCA}.

\paragraph{Notation and Definitions}
Throughout the paper, 
bold lowercase variables denote vectors (\eg $\vx$) or vector-valued random variables (\eg $\rvx$),
bold uppercase are used for matrices (\eg $\mX$) or matrix-valued random variables (\eg $\rmX$) and
unbold lowercase are scalars (\eg $x$) or random variables (\eg $\rx$). 
The transpose of a matrix is denoted as $\mA^{\top}$ and
$\ve^{(i)} = [0,\dots,0,1,0,\dots,0]$ is the standard basis vector with a 1 at $i$th position. 
There are $M$ views in total and subscripts are intended to identify the view-specific variable, (\eg $\rvx\mi, \mSigma_{mm}$), which is different from an element of a vector that is specified by subscript (\eg $\rx_{mi}$).
The difference should be clear from context.

\section{Probabilistic CCA} \label{sec:PCCA}

Canonical correlation analysis (CCA) \citep{hotelling1992CCA} is a classical subspace learning method that extracts information from the cross-correlation between two variables.
Let  $\rvz_1 \in \RR^{d_1}$ and $\rvz_2 \in \RR^{d_2}$ be a pair of random vectors corresponding to two different views. CCA linearly projects these onto the subspace $\RR^{d_0}$ as
$\rvr_1 = \mU_1^{\top} \zvec_1$ and $\rvr_2 = \mU_2^{\top} \zvec_2$, where $\mU_1 \in \RR^{d_1 \times d_0}$ and $\mU_2 \in \RR^{d_2 \times d_0}$ and $ 0 < d_0 \le \min\{d_1, d2\}$, 
such that each pair of components $ (\rvr_1(i), \rvr_2(j))$ are maximally correlated if $ i=j$ and uncorrelated otherwise. 
Let $(\vmu_1, \mSigma_{11})$ and $(\vmu_2, \mSigma_{22})$ be the mean and covariance matrices of $\rvz_1$ and $\rvz_2$, respectively, and $\mSigma_{12}$ is their cross-covariance. Then CCA can be formulated as the optimization problem
\begin{align}
&\max_{\mU_1, \mU_2} \tr [ \mU_1^{\top} \mSigma_{12} \mU_2 ]\\
&\mU_1^{\top} \mSigma_{11} \mU_1 = \mU_2^{\top} \mSigma_{22} \mU_2 = \eye_{d_0} \nonumber
\end{align}
Given $ (\vv_{1i}, \vv_{2i}), ~ i \in [1, ..., d_0]$ as the pairs of left and right singular vectors corresponding to $d_0$ largest singular values, $p_i ~ i \in [1, ..., d_0]$, of the correlation matrix 
$\mSigma_{11}^{-1/2} \mSigma_{12} \mSigma_{22}^{-1/2}$, 
 the solution to the CCA problem is given by $ (\vu_{1i}, \vu_{2i} ) =  (\mSigma_{11}^{-1/2} \vv_{1i},  \mSigma_{22}^{-1/2} \vv_{2i} ), ~ i \in [1, ..., d_0]$ where $ (\vu_{1i}, \vu_{2i} )$, also called \textit{canonical pairs of directions}, form the columns of $(\mU_{1}, \mU_{2} )$ and $\mP_{d_0}=\text{diag}([p_0, ..., p_{d_0}])$ is the diagonal \textit{matrix of canonical correlations}.

\citet{bachPCCA} and \citet{browne1979IBFA}  proposed a probabilistic generative interpretation to the classical CCA problem
that reveals the shared latent representation explicitly. 
An extension of their results to a more flexible model can be expressed as follows.

\begin{figure*}[t]
	\begin{minipage}[c]{0.60\textwidth}
		{\includegraphics[width=0.95\textwidth]{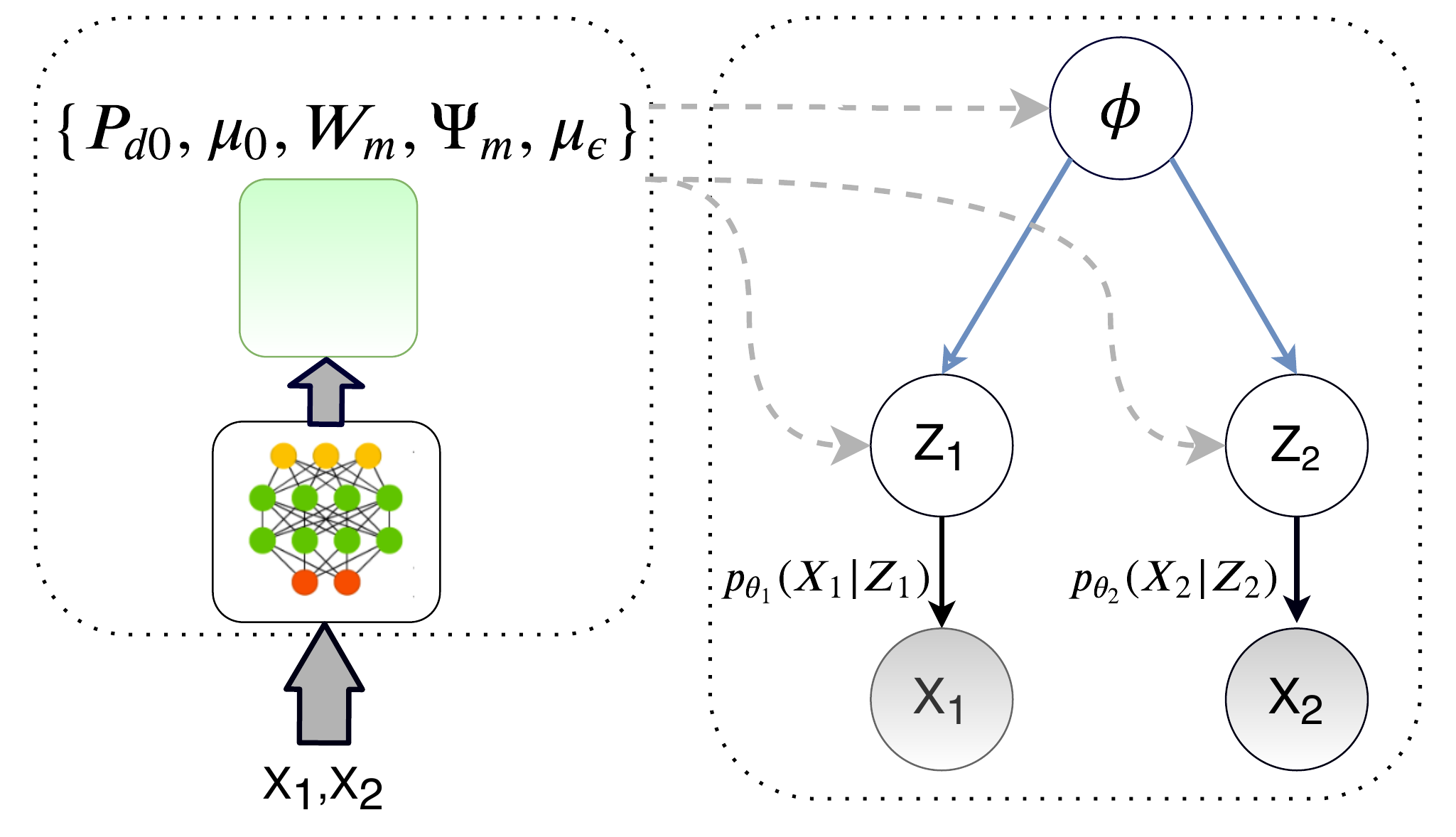}}
	\end{minipage}
	\begin{minipage}[c]{0.01\textwidth}
		\text{}
	\end{minipage}
	\begin{minipage}[c]{0.39\textwidth}
		\caption[Graphical representation of the deep probabilistic multi-view model]{
			Graphical representation of the deep probabilistic CCA model,
			where the blue edges belong to latent linear probabilistic CCA  model
			and the black edges represent the deep nonlinear observation networks (decoders)
			$p_{\theta\mi }(\rvx\mi \given \rvz\mi ) = g\mi (\rvz\mi; \theta\mi )$.
			{Shaded nodes denotes observed views and dashed line represent the stochastic samples drawn from the approximate posteriors.}
		} \label{fig:graphical_PCCA} 
	\end{minipage}
\vskip -.1in
\end{figure*}

\begin{theorem} \label{thm:PCCA}
	Assume the probabilistic generative model for the graphical model in Figure \ref{fig:graphical_PCCA} as:
	\begin{align} \label{eqn:gPCCA}
	\statevec  &\sim \normal(\vmu_0, \eye_{d_0}), \quad 0 < d_0 \le \min\{d_1, d_2\} \\  	
	\rvz_1 \given  \statevec &\sim \normal(\mW_1 \statevec + \vmu_{\eps_1}, \mPsi_1),
	\quad \mW_1 \in \RR^{d_1 \times d_0}, \mPsi_1 \succcurlyeq 0 \nonumber\\  	
	\rvz_2 \given  \statevec &\sim \normal(\mW_2 \statevec + \vmu_{\eps_2}, \mPsi_2),
	\quad \mW_2 \in \RR^{d_2 \times d_0}, \mPsi_2 \succcurlyeq 0 \nonumber
	\end{align}
	where $\statevec $ is the shared latent representation.
	The maximum likelihood estimate of the parameters of this model can be expressed in terms of the canonical correlation directions as
	\begin{align} \label{eqn:ml_est_PCCA}
	\hat{\mW}_1  &= \mSigma_{11} \mU_{1} \mM_1 \\
	\hat{\mW}_2  &= \mSigma_{22} \mU_{2} \mM_2 \nonumber \\
	\hat{\mPsi}_1  &= \mSigma_{11} - \hat{\mW}_1 \hat{\mW}_1^{\top} \nonumber \\
	\hat{\mPsi}_2  &= \mSigma_{22} - \hat{\mW}_2 \hat{\mW}_2^{\top} \nonumber \\
	\hat{\vmu}_{\eps_1} &= \vmu_1 - \hat{\mW}_1 \vmu_0 \nonumber \\
	\hat{\vmu}_{\eps_2} &= \vmu_2 - \hat{\mW}_2 \vmu_0 \nonumber 
	\end{align}
	where $\mM_1, \mM_2 \in \RR^{d_0 \times d_0} $ are arbitrary matrices with spectral norms smaller than one that satisfy $\mM_1 \mM_2^{\top} = \mP_{d_0}$,
	and the residual errors terms can be defined as  $\rveps_1 := \rvz_1 - \mW_1 \statevec$ and $\rveps_2 := \rvz_2 - \mW_2 \statevec$.  
	This probabilistic graphical model induces conditional independence of $\rvz_1$ and $\rvz_2$ given $\statevec$.
	The parameter $ \vmu_0$ is not identifiable by maximum likelihood. 

\end{theorem}  
\begin{proof} 	See Appendix \ref{apdx:proof_PCCA}.    \end{proof}

In contrast to the results in \citep{bachPCCA} where $\vmu_0 = \mathbf{0} $,  
here we introduce $\vmu_0$ as an extra degree of freedom. We will see in the following 
the important role of this parameter
in optimizing the upper bound on the likelihood,
and also in the inference of the shared representation of deep probabilistic CCA.
We will also derive an analytical form to identify it based on the parameters of the probabilistic multi-view layer, 
and develop an identification method for the arbitrary matrices $\mM_1, \mM_2$.

\subsection{Generalization to arbitrary number of views} \label{sec:general_MV}
As an extension to an arbitrary number of views for probabilistic CCA, \citep{archambeau2009sparse} proposed a general probabilistic model as follows:
\begin{align} \label{eqn:gMBFA}
\rvz\mi &= \mW\mi \statevec_0 + \mV\mi \statevec\mi + \vmu\mi + \nu\mi, \\
\nu\mi & \sim \normal(\zerovec, \tau\mi^{-1} \eye_{d\mi}), \nonumber \\
& \quad \mW\mi \in \RR^{d\mi \times d_0},  \mV\mi \in \RR^{d\mi \times q\mi }, 
\forall m \in \{1, ..., M\} \nonumber
\end{align}
where $ \{\vmu\mi\}_{m=1}^{M}$ and $ \{\nu\mi\}_{m=1}^{M}$ are the view specific offsets and residual errors, respectively. 
This model  can also be viewed as multibattery factor analysis (MBFA) \citep{klami2014GroupFactorAnalysis, browne1980MBFA} in the statistics literature,
which describes the statistical dependence between all the views by a single latent vector, $\statevec_0$, and explains away the view-specific variations by factors private to each view.
Moreover, limiting to single view, 
if the prior on the view-specific factor is multivariate independent Gaussian, 
this model includes the probabilistic factor analysis as a special case, 
and also reduces to probabilistic PCA given the prior is isotropic as well. \citet{archambeau2009sparse} followed a Bayesian approach to the linear generative model \eq{eqn:gMBFA} and proposed a variational Expectation-Maximization algorithm to estimate the model parameters.

A reformulation for the parameters of this general model inspired by the maximum likelihood solutions of probabilistic CCA in Theorem \ref{thm:PCCA} is presented in Appendix \ref{apdx:GPMV}.

Although constraining the observation models to the classical linear model \eq{eqn:gPCCA} offers closed form inference for the latent variable(s) and efficient training algorithms the resulting power  is very limited in modeling increasingly complex data distribution.
On the other hand, the generative descriptions of the probabilistic models in general, and the probabilistic multi-view models \eq{eqn:gPCCA} and \eq{eqn:gMBFA} in particular, can be extended naturally as the building blocks of more complex hierarchical models \citep{klami2013bayesianCCA}.
A well established method to increase the capacity and improve the expressiveness of such models is
using deep neural networks to capture nonlinear and complex structures in data distribution,
therefore one can append deep generative networks on top of the linear probabilistic model to obtain a
combined model denoted in this work as \emph{deep probabilistic CCA} or \emph{deep probabilistic multi-view} network in general.
A graphical representation of this model is depicted  in Figure \ref{fig:graphical_PCCA}.
Let $\rvx := \{\rvx\mi \in \RR^{d'\mi }\}_{m=1}^{M}$ denote the collection of observations of all views and
$\rvz := \{\statevec \in \RR^{d_0}\} \cup \{\rvz\mi \in \RR^{d\mi }\}_{m=1}^{M}$ be the collection of  the shared latent representation and latent variables corresponding to each view.
The \textit{latent linear probabilistic CCA layer} of the form presented in \eq{eqn:gPCCA} 
models the linear cross-correlation between all latent variables
$\{\rvz\mi \}_{m=1}^{M}$ 
in the latent space
while the nonlinear generative \textit{observation networks},
also called the \textit{decoders} in the context of variational auto-encoders, are responsible for expressing the  complex variations of each view.
The observation models are described by deep neural networks
$p_{\theta\mi }(\rvx\mi \given \rvz\mi ) = g\mi (\rvz\mi; \theta\mi )$
with the set of model parameters $\theta= \{\theta\mi \}_{m=1}^{M}$.
In the following, an approximate variational inference approach is presented for training of this deep generative multi-view model.

\subsection{Variational inference} \label{sec:var_inf}
To obtain the maximum likelihood estimate of the model parameters,  it is desirable to maximize the marginal data log-likelihood averaged on the dataset $\mathcal{D} = \{x^{(i)}\},~ i=1,..,N$
\begin{align*} 
\log p_{\theta}(\rmX) 
= \frac{1}{N} \sum_{i=1}^{N} \log p_{\theta}(\vx^{(i)})
& \simeq  \E_{\rvx\sim \hat{P}_{data}} [ \log p_{\theta}(\rvx) ] 
\end{align*}

This objective requires marginalization over all the latent variables which entails the expectation of likelihood function $p_{\theta}(\rvx \given \rvz)$ over
the prior distribution of the set of latent variables, $p(\rvz)$.
The marginalization is typically intractable for flexible models, hence,
one work around is to follow the variational inference principle \citep{jordan1999intro_variational},
by introducing an approximate posterior distribution $q_{\eta}(\rvz \given \rvx)$ 
--- also known as \textit{variational inference network} in the context of \textit{amortized variational inference} --- 
and maximize a resulting variational lower bound on the marginal log-likelihood.
This approach  has   recently attained renewed interest and studied extensively, and is considered as a default, 
flexible statistical inference method.
 Approximate variational inference is often modeled by deep NNs with model parameters $\eta$.
Therefore, we obtain the variational lower bound \citep{rezende2014stochastic}
\begin{align} \label{eqn:ELBO0}
\log p_{\theta}(\rvx) \ge  \E_{q_{\eta}} [ \log p_{\theta}(\rvx \given \rvz) ]- \KL [ q_{\eta}(\rvz \given \rvx) \Vert p(\rvz) ] 
\end{align}
This bound, also known as \emph{evidence lower bound (ELBO)}, can be decomposed into two main terms:
first, the expectation of the log-likelihood function $\log p_{\theta}(\rvx \given \rvz)$, known as the \textit{negative reconstruction error}.
The conditional independence structure of the deep generative multimodal model implies that the likelihood function can be factored hence the negative reconstruction error can be expressed as
\begin{align*} 
\E_{q_{\eta}} [ \log p_{\theta}(\rvx \given \rvz) ] =
\sum_{m=1}^{M}\E_{q_{\eta}} [ \log p_{\theta\mi }(\rvx\mi \given \rvz\mi ) ] .\nonumber
\end{align*}
Although the expectations in above do not typically provide a closed analytical form, it can be approximated using Monte Carlo estimation by drawing $L$ random samples from the approximate posterior $q_{\eta}(\rvz \given \rvx)$ for each data point $\rvx = \vx^{(i)}$
\footnote{This, indeed, leads to the Monte Carlo approximation of the gradient of the expected log-likelihood required for the stochastic gradient descent training \citep{rezende2014stochastic}}.

The second term in ELBO is the \textit{KL divergence} between the approximate posterior and the prior distribution of the latent variables, which acts as a regularizer injecting prior knowledge about the latent variable into the learning algorithm.
Taking into account the conditional independence of the latent variables
$\{\rvz\mi \given \statevec \}$ induced by the probabilistic graphical model of latent linear layer \eq{eqn:gPCCA}, the approximate posterior of the set of latent variables can be factorized  as
$ q_{\eta}(\rvz \given \rvx) = q_{\eta}(\statevec \given  \rvx)
\prod_{m=1}^{M} q_{\eta}(\rvz\mi \given \statevec, \rvx)$
therefore, the KL divergence term can be decomposed to
\begin{align} \label{eqn:kl}
\KL [ q_{\eta}(\rvz \given \rvx) \Vert p(\rvz) ] = &
\KL [ q_{\eta}(\statevec \given \rvx) \Vert p(\statevec) ] +
\nonumber\\ &
\sum_{m=1}^{M}\KL [ q_{\eta}(\rveps\mi \given \rvx) \Vert p(\rveps\mi ) ]
\end{align}
More details can be found in appendix \ref{apdx:proof_mean_PCCA}.

We model the variational approximate posteriors by joint multivariate Gaussian distributions
with marginal densities 
$q_{\eta}(\rvz\mi \given \rvx\mi ) = \normal(\rvz\mi; \vmu\mi(\rvx\mi ), \mSigma_{mm}(\rvx\mi ))$,
that are assumed to be elementwise independent per each view for simplicity
so having diagonal covariance matrices
$\mSigma_{mm} = \text{diag}(\vsigma\mi^2(\rvx\mi ))$, $\vsigma_{m} \in  \RR^{d\mi }$
and the cross correlation specified by canonical correlation matrix
$\mP_{d_0}=\text{diag}(\vp(\rvx)), \vp \in  \RR^{d_0}$.
The parameters of these variational posteriors are specified by separate deep neural networks, also called \textit{encoders}.
In this model a set of encoders are used to output the view-specific moments $ \{(\vmu\mi, \vsigma\mi^2) = f\mi (\rvx\mi; \eta\mi )\}_{m=1}^{M}$,
and an encoder network describes the cross correlation
$\vp = f_0(\rvx^*; \eta_0)$,
whereas, depending on the application, $\rvx^*$ can be either one (or a subset) of the views, 
when only one (or a subset) of the views are available at the test time, 
\eg the multi-view setting where $\rvx^*= \rvx_1$,
or concatenation of all the views \eg in multi-modal setting.
Altogether, the inference model is parameterized by $\eta = \{\eta_0\} \cup \{\eta\mi \}_{m=1}^M$.
Having obtained the moments of approximate posteriors, we can obtain the canonical directions and subsequently the parameters of the probabilistic CCA model,
according to the results presented in theorem \ref{thm:PCCA}.
It is worth noting that the diagonal choices for covariance matrices $\{ \mSigma_{mm} \}_{m=1}^M$,
simplify the algebraic operations significantly, 
resulting in trivial SVD computation and matrix inversion
required in theorem \ref{thm:PCCA}. \footnote{These types of simplifying assumption on the approximate posteriors have also been used in various deep variational inference models \citep{rezende2014stochastic, kingma2013VAE}.
	Although the representation power of such linear latent model is limited but using flexible enough deep generative models, that can explain away the complex nonlinear structures among the data, can justify these choices.}
Consequently, one can also readily verify that
the canonical pairs of directions will be
$ (\vu_{1i}, \vu_{2i} ) = (\evsigma_{1i}^{-1/2}  \ve^{(i)}, \evsigma_{2i}^{-1/2} \ve^{(i)})$ 
where $\ve^{(i)}$ is the standard basis vector $[0,\dots,0,1,0,\dots,0]$ with a 1 at  $i$th position.

We assume isotropic multivariate Gaussian priors on the latent variables as 
$\statevec \sim \normal(\mathbf{0}, \lambda_{0}^{-1} \eye)$,
$\rveps\mi \sim \normal(\mathbf{0}, \lambda\mi^{-1} \eye)$
and specify the approximate posteriors by Gaussian distributions, as explained above,
the KL divergence terms can be computed in closed forms   \citep{kingma2013VAE}.
In the following, we provide an analytical approach to optimally identify  the mean of shared latent variable, $\vmu_{0}$, from the parameters of the model  and drive the optimal solution for $\mM_1, \mM_2$.
\begin{lemma} \label{thm:mean_PCCA}
	I) Rewriting the KL divergences with respect to the terms depending on the mean of latent factors give rise to the following optimization problem
	\begin{align} \label{eqn:opt_meanPhi}
	\min_{\vmu_{0}} & \frac{1}{2}  \lambda_0 \| \vmu_{0} \|^2 +
	\frac{1}{2} \sum_{m=1}^{M} \lambda\mi \| \vmu_{\eps\mi } \|^2 + \mathcal{K}  \\
	\text{s.t. }
	\vmu_{\eps\mi } &= \vmu\mi - \mW\mi \vmu_0, \quad \forall m \in \{1,..., M\}  \nonumber
	\end{align}
	where $\mathcal{K}$ is sum of the terms not depending on the means.
	Solving this optimization problem results the optimal minimizer
	\begin{align} \label{eqn:meanPhi}
	\vmu_{0}^* = (\lambda_0 \eye + \sum_{m=1}^{M} \lambda\mi \mW \mi^{\top} \mW \mi )^{-1} (\sum_{m=1}^{M} \lambda\mi \mW\mi^{\top} \vmu \mi ).
	\end{align}
	Having obtained the optimal $\vmu_0^*$, one can compute the means of the view-specific factors,  $\{\vmu_{\eps\mi }\}_{m=1}^M$, subsequently.\\
	II) Given similar prior distributions on all view-specific factors, $\{\rveps\mi\}_{\mi=1}^{M}$, the solutions that minimize the KL divergence term are $\mM \mi = \mM, \forall m \in {1, ..., M}$ where the matrix $\mM$ is the square root of matrix $\mP_{d_0}$, \ie 
	$\mM = \mP_{d_0}^{1/2} \mR$ and $\mR$ is an arbitrary rotation matrix.
\end{lemma}  
\begin{proof} 	See Appendix \ref{apdx:proof_mean_PCCA} for the proof.
\end{proof}

According to the inference network, the optimal $\vmu_{0}$ obtained by \eq{eqn:meanPhi} is a function of all the views
that can be viewed as a type of data fusion in the deep space, making it an appropriate choice for the multi-modal setting.
On the other hand, in multi-view setting we are interested in the solution that depends only on the primary view available at the test time.
To deal with this, we can solve a revised version of the optimization problem \eq{eqn:opt_meanPhi} by ignoring the terms that are depending on the non-primary views
which leads to the minimizer 
\begin{align} \label{eqn:meanPhi_1}
\hat{\vmu}_{0} = (\lambda_0 \eye + \lambda_1 \mW_1^{\top} \mW_1 )^{-1} \lambda_1 \mW_1^{\top} \vmu_1 .
\end{align}

\paragraph{Remark} As an alternative approach in the multi-view setting, one can train the model using the optimal inference based on both views in equation \eq{eqn:meanPhi} while using the primary view-based estimate $\hat{\vmu}_{0}$ in \eq{eqn:meanPhi_1} at the test time, 
but our empirical studies showed that using the same inference as in \eq{eqn:meanPhi_1}
for both training and test time offers richer shared representation variable resulting in slightly better performance in the downstream  tasks.

\paragraph{Remark} Another possible approach is to treat $\vmu_{0}$ as an extra parameter that is directly inferred by a deep NN, but this needs more NN layers to train and in practice we found this approach less efficient than the proposed optimal procedure.

We further assume that the rotation matrix $\mR$ is identity in the solutions to the probabilistic linear models \eq{eqn:ml_est_PCCA}, 
while leaving it to the deep generative network to approximate the rotation. Specifically, in our neural network architecture we select a fully connected as the first layer of the decoder to exactly mimic the rotation matrix.

In summary, the encoders together with the parameterizations of the model in \eq{eqn:ml_est_PCCA} 
provide variational inference network for the parameters of the \textit{latent probabilistic multi-view  model},
$\{\mP_{d_0}(\rvx_1), \vmu_{0}, \mW\mi (\rvx\mi ), \mPsi\mi (\rvx\mi ), \vmu_{\eps\mi } \}_{m=1}^M$,
as non-linear functions of the observations.

\paragraph{Drawing samples from the latent variables:}
Given the variational parameters of the latent probabilistic CCA model,
one can draw samples of the latent factors $\{\statevec, \rveps_1, \rveps_2\}$ 
from the approximate posteriors $\{ q_{\eta}(\statevec ), q_{\eta}(\rveps_1 ), q_{\eta}(\rveps_2 ) \}$,
using a differentiable transformation based on the reparameterization trick \citep{kingma2019introduction_VAE},
and generate latent representations as 
$ \rvz_1 =  \mW_1 \statevec + \rveps_1$ 
and 
$\rvz_2 =  \mW_2 \statevec + \rveps_2$, 
which are fed into the decoders
to generate samples $ \{\hat{\rvx}_1, \hat{\rvx}_2 \} $ at the observation space.
This procedure produces latent samples that satisfy the conditional independence rule of the probabilistic CCA while being cross correlated as specified by the variational canonical correlation $\mP_{d_0}$. 
Therefore, the reconstruction error term can be stated as 
\begin{align*} 
\E_{q_{\eta}} & [ \log p_{\theta}(\rvx \given \rvz) ]  = \nonumber \\
& \E_{q_{\eta}(\statevec ), q_{\eta}(\rveps_1 )} [ \log p_{\theta_1}(\rvx_1 \given \rvz_1=  \mW_1 \statevec + \rveps_1) ] + \nonumber \\
& \E_{q_{\eta}(\statevec ), q_{\eta}(\rveps_2 )} [ \log p_{\theta_2}(\rvx_2 \given \rvz_2=  \mW_2 \statevec + \rveps_2) ].\nonumber 
\end{align*}

\subsection{Related work}

To capture nonlinearity in the multi-view data, several kernel-based methods have been proposed \citep{Hardoon_2004_KCCA,Bach_2003_KIC}.
Kernel-based methods, in general,  require large memory to store a massive amount of training data to use in the test phase and in particular kernel-CCA requires an $N \times N$ eigenvalue decomposition which is computationally expensive for large datasets.
To overcome this issue, some kernel approximation techniques based on random sampling of training data are proposed in \citep{williams2001nystromKCCA} and \citep{lopez2014randomizedKCCA}.
Moreover, the probabilistic non-linear multi-view learning are considered in \citep{shon2006learning, damianou2012manifold}.
As an alternative, deep neural networks (DNNs) offer powerful  
parametric models that can be trained for large pools of data  using the recent advances of the stochastic optimization algorithms.  
In the multi-view setting, a deep auto-encoder model, called (SplitAE), was designed in \citep{ngiam2011multimodalDL} in which an encoder maps the primary view to the latent representation and two encoders are trained so that the reconstruction error of both views are minimized.
On the other hand, the classical CCA is extended to deep CCA (DCCA) in \citep{andrew2013DCCA} by replacing the linear transformations of both views with two deep nonlinear NNs and then learning the model parameters by maximizing the cross correlation between the nonlinear projections.
DCCA is then extended to deep CCA autoencoder (DCCAE) in \citep{wang2015deepMV} where it leverages autoencoders to additionally reconstruct the inputs hence intruding extra reconstruction error terms to the objective function. 
While DCCAE could improve the representation learning over the DCCA,
empirical studies showed that it tends to ignore the added reconstruction error terms
resulting in poor reconstructed
views \cite{wang2015deepMV}.
The training algorithms of such classical CCA-based methods require sufficiently large batch size to approximate the covariance matrices and the gradients.
Moreover, they also do not naturally provide an inference model to estimate shared latent factor 
and do not enable generating sampling from the model in the input space 
while also being restricted to the two-view setting.
In contrast, the reconstruction error terms appear naturally in the objective function of the variational inference, the ELBO, so play a fundamental role in the training
and thus richer decoder and reconstruction are expected using the proposed variational autoencoders.
Furthermore, the stochastic backpropagation method  with small mini-batches has proven as a standard and scalable technique for training deep variational autoencoders \citep{rezende2014stochastic}. 
Moreover, the probabilistic multi-view model 
enables enforcing desired structures such as sparsity \citep{archambeau2009sparse}
by adopting a broader range of exponential family distributions for priors and approximate posteriors on the latent factors to capture
while this property is not immediately apparent in the classical CCA-based variants.

More recently, a variational two-view autoencoder was proposed in \citep{wang2017acousticVCCA, wang2016_deepVCCA}  
that in principle offers a generative two-view model with shared representation or 
shared + view-specific factors.
Despite the name of the method, in theory these works fail to draw connections between the proposed two-view models with the canonical directions and the probabilistic CCA interprtation in \cite{bachPCCA} 
while the inference is not customized beyond the black box variational inference   
which can explain why they perform weaker than DCCAE in some experimental studies.   
On the other hand, in comparison to the deep probabilistic CCA proposed here, the shared latent representation 
equally contributes in both views, 
so these variational two-view methods can be viewed as special cases of the more generic model proposed here when $\mW_1=\mW_2=\eye$
hence, they are expected to offer lower flexibility in modeling two-view dataset compared to the deep probabilistic CCA.

\paragraph{Remark} 
Normalizing flows is a technique to specify a flexible and arbitrarily complex distribution 
by applying a sequence of invertible transformations on a simple base distribution
\citep{rezende_shakir2015_variationalInfWithNormalizingFlows}.
In the deep probabilistic multi-view model, we can also obtain a more complex approximate posterior by applying a rich normalizing flow, such as those in \citep{karami2019invertibleConF}, on the Gaussian distributed latent variables $\{\rvx\mi \}_{m=1}^M$ generated by latent linear probabilistic multi-view layer.
By reducing the gap between the true posterior and its approximate,
this technique is expected to provide a  more expressive generative model for complex multi-view applications hence serving as a potential candidate for future studies.

\section{Experiments} \label{sec:Experiments}
We empirically evaluate the representation learning performance of the proposed method and compare against some well established baseline algorithms on two scenarios: 
I) when all views are available at training but only a single view (the primary view) is available at the test time, namely multi-view setting, and
II) all the views are available at the train and test time, namely multi-modal setting.

\subsection{Multi-view experiments} \label{sec:exp_MNIST}

\begin{table}[t!] 
	\centering
	\begin{tabular}{l|c | c |c}
	Method 									& \small{Error $(\%) $}  & \small{NMI $(\%) $}  & \small{ACC $(\%) $}  \\
	\toprule
	\textbf{Linear CCA }		& 	19.6 &                 56.0                 &                 72.9                 \\ \hline
	\textbf{SpliAE }				& 	11.9 &                 69.0                 &                 64.0                 \\ \hline
	\textbf{KCCA 	}				&  	5.1 &                 87.3                 &                 94.7                  \\ \hline
	\textbf{DCCA }				& 		2.9		&                   92.0                &                 97.0                  \\ \hline
	\textbf{DCCAE }				& 2.2 &                             93.4      &                 97.5                  \\ \hline 
	\textbf{VCCA}							& 3.0 &                 -                  &                 -                  \\ \hline
	\textbf{VCCA-private }				& 2.4 &                 -                  &                 -                  \\ \hline 
	\hline
	\textbf{VPCCA}				& \textbf{1.9} &             \textbf{94.8}             &             \textbf{98.1}             \\ \hline 
	\bottomrule
\end{tabular}
	\caption[Classification and clustering performance of multi-view learning algorithms on noisy two-view MNIST]{\small Performance of the downstream tasks for different multi-view learning algorithms on the noisy two-view MNIST digit images.
		Performance measures are classification error rate (the lower the better), normalized mutual information (NMI) and accuracy (ACC) of clustering (the higher the better) \citep{cai2005documentclustering}. 
		\textbf{VPCCA}: multi-view setting, \ie only primary view is available at the test time so $\vmu_{0}$ of equation \eq{eqn:meanPhi_1} is used.
		The results of variational PCCA method are averaged over 3 rials where the results of the baseline methods are from \citep{wang2015deepMV,wang2016_deepVCCA}. 
		The baseline methods are
		\textbf{Linear CCA}: linear single layer CCA,
		\textbf{DCCA}: deep CCA \citep{andrew2013DCCA},
		\textbf{Randomized KCCA}: randomized kernel CCA approximation with Gaussian RBF kernels and  random Fourier features  \citep{lopez2014randomizedKCCA},
		\textbf{DCCAE}: deep CCA-Auto encoder \citep{wang2015deepMV},
		\textbf{VCCA}:  multi-view variational auto-encoder  \citep{wang2016_deepVCCA} 
		\textbf{VCCA-private}:  shared-private multi-view variational auto-encoder  \citep{wang2016_deepVCCA}.
	}
	\label{tbl:classERR}
\end{table}

\paragraph{\textbf{Experimental design}:}
For the experimental study, we used the two-view noisy MNIST datasets of \citep{wang2015deepMV} and \cite{wang2016_deepVCCA} created based on MNIST handwritten digits that consists of grayscale images of size $28 \times 28$ pixels with pixel values scaled to range $[0, 1]$.
The first view of  the dataset is synthesized  by  rotating each image at angles randomly sampled from uniform distribution  $\mathcal{U}(-\pi/4,\pi/4)$ while
the second view is randomly sampled from the images with similar identity to the first view but not necessary the same image and then is corrupted by random uniform noise while the final value is truncated to remain in range  $[0, 1]$.
As a result of this procedure, both views are just sharing the same identity (label)  of the digit but not the style of the handwriting as they are from arbitrary images in the same class. 
The training set is divided into training/validation subsets of length $50K/10K$ and the performance is measured on the $10K$ images in the test set.

To make a fair comparison, we used neural network architectures with the same capacity as those  used in \citep{wang2015deepMV} and \citep{wang2016_deepVCCA}.
Accordingly, for the deep network models, all the inference networks and decoding networks are composed of 3 fully connected nonlinear hidden layers of size $1024$ units, 
where \texttt{ReLU} gate is used as nonlinearity for all the  hidden units of the deep networks. 
The first and the second encoder specify $ (\vmu_1, \text{diag}(\evsigma_{1i}^2)) = f_1(\rvx_1; \theta_1)$, $ (\vmu_2, \text{diag}(\evsigma_{2i}^2)) = f_2(\rvx_2; \theta_2)$ 
where the variances are specified by \texttt{softplus} function, 
and an extra encoder models the canonical correlations  $\text{diag}(p_i) $ using the 
\texttt{sigmoid} function as the output gate.
Independent Bernoulli distributions and independent Gaussian distributions are selected 
to specify the likelihood functions of the first and the second view, respectively, 
with the parameters of each view being specified by its own decoder network;
the \texttt{sigmoid} functions are applied on the outputs to estimate the means of both views while the variances of the Gaussian variables are specified by \texttt{softplus} functions. 
In order to prevent over-fitting, stochastic drop-out \cite{srivastava2014dropout} is  applied  to all the layers as a regularization technique.
ADAM optimizer \citep{kingma2014adam} is adopted  for training of the parameters of the deep neural networks. 
The details of the experimental setup can be found in Appendix \ref{apdx:modelArch}.

To evaluate learned representation, the discriminative task and clustering task are examined on the shared latent variable.
For the discriminative goal, the one-versus-one linear SVM classification algorithm is applied on the shared representation $\statevec$. The parameters of the SVM algorithm are tuned using the  validation set and the classification error is measured on the test set.
We have also performed spectral clustering algorithm \citep{vonLuxburg2007tutorialSpectralClustering} on the \textit{k}-nearest-neighbor graph constructed from the shared representation.
To comply with the experiments in \citep{wang2015deepMV} the degree (number of neighbors) of the nodes is tuned in the set $ \{5, 10, 20, 30, 50 \}$ using the validation set,
and finally, it uses \textit{k}-means as the last step to construct final partitioning into $10$ clusters in the embedding space.
The proposed deep probabilistic CCA is compared against the available multi-view methods 
in terms of the performance of the downstream tasks reported in Table \ref{tbl:classERR},
where the results highlight that proposed variational model significantly improves the 
representation learning in multi-view datasets. 

2D embeddings of the shared latent representations using t-SNE is also visualised in Appendix \ref{apdx:tsne}, which quantitatively verify that the learned features of the images of different classes are well separated.

Repeating the experiments on multi modal setting (\ie both views are available at the test time) and using the equation \eq{eqn:meanPhi} to recover the mean of the shared latent variable can significantly improve the performance of downstream tasks resulting in classification error=$0.4\%$ and clustering NMI=$98.3\%$ or ACC=$99.4\%$,
confirming the merit of the proposed algorithm in successfully integrating information from different modality.

\subsection{Multi-modal clustering}
An important and interesting application of the proposed deep generative model is in clustering multi-modal datasets which we evaluate in this set of experiments.
Recently, a deep multi-modal subspace clustering \citep{abavisani2018multimodal} has been successfully extended the idea of 
deep subspcae clustering (DSC) \citep{ji2017deepSupspaceClustering} into multiple modalities.
The key component of such approaches is applying a self-expressive layer on a non-linear mapping of the data obtained by deep auto-encoders, which represent the projection of the data points as a linear combination of other data points projections.
Although offering significant improvement in clustering performance for datasets lying in non-linear subspaces, 
such methods require a self-representation coefficient matrix of size $N \times N $ where $N$ is the number of data points which makes them prohibitively expensive for large datasets. 

\subsubsection{Datesets}
The clustering performances of the proposed method are evaluated  over the following standard datasets.
Samples from all modalities of these datasets are depicted in Figure \ref{fig:samples_dataset}. 

\begin{figure*}[t!]\centering
	\newcommand{\figpath}{\figtablepath/}
	\begin{minipage}{.4\linewidth}
	\centering
		\subfigure[]{
			{\label{fig:samples_MNIST}
				\includegraphics[width=0.60\textwidth]{\figpath MNIST_samples.png}}}
		\subfigure[]{
			{\label{fig:samples_USPS}
				\includegraphics[width=0.60\textwidth]{\figpath USPS_samples.png}}}
		\subfigure[]{
			{\label{fig:samples_YaleB}
				\includegraphics[width=.60\textwidth]{\figpath YaleB_samples.png}}}
	\vspace{-.1in}
	\caption[Sample images form multi-modal datasets]{\footnotesize \label{fig:samples_dataset}
		(a) Sample images form MNIST dataset,
		(b) and their corresponding samples from the second modality, drawn from USPS datasets.
		(c) Sample images from faces and face components from Extended Yale-B dataset; modalities are showed in different rows.
	}
	\end{minipage}
	\hfill
    \begin{minipage}{.57\linewidth}
    \centering
        \subfigure[]{
			{\label{tbl:clusterin_multimodal}
				\begin{tabular}{c|ccc|ccc}
	& \multicolumn{3}{c}{\textbf{Digits}} & \multicolumn{3}{c}{\textbf{Extended Yale-B}} \\
	&  ACC  &  NMI  &         ARI         &  ACC  &  NMI  &             ARI              \\ \toprule
	\textbf{CMVFC}      & 47.6  & 73.56 &        38.12        & 66.84 & 72.03 &              40              \\ \hline
	\textbf{TM-MSC}      & 80.65 & 83.44 &        75.67        & 63.12 & 67.06 &            38.37             \\ \hline
	\textbf{MSSC}       & 81.65 & 85.33 &        77.36        & 80.3  & 82.78 &            50.18             \\ \hline
	\textbf{MLRR}       & 80.6  & 84.13 &        76.53        & 67.62 & 73.36 &            40.85             \\ \hline
	\textbf{KMSSC}      & 84.4  & 89.45 &        79.61        & 87.65 & 81.5  &            63.83             \\ \hline
	\textbf{KMLRR}      & 86.85 & 80.34 &        82.76        & 82.45 & 85.43 &            59.71             \\ \hline
	\textbf{DMSC} & 95.15 & 92.09 &        90.22        & 99.22 & 98.89 &            98.38             \\ \hline
 	\textbf{VPCCA} & \textbf{98.78}& \textbf{96.72}& \textbf{97.35}& \textbf{99.71}& \textbf{99.52}& \textbf{99.19}\\	 
 	\hline\hline
\end{tabular}}
		}
		\caption[Performance for different multi-modal clustering algorithms]{\footnotesize
		\label{tbl:clustering}
		Performance for different multi-modal clustering algorithms on two-modal  handwritten digits made from MNIST and USPS 
		and multi-modal facial components extracted from Yale-B dataset. 
		Performance measures are clustering accuracy rate (ACC), normalized mutual information (NMI) and adjusted Rand index (ARI),  all measures are in percent and the higher means the better.
		Here, we assume all modalities are available at the test time so \textbf{VPCCA} uses $\vmu_{0}$ of equation \eq{eqn:meanPhi} is used.
		The results of variational PCCA method are averaged over 3 rials.
		The baseline subspace clustering methods are
		TM-MSC \citep{zhang2015TMMSC}, CMVFC \citep{cao2015constrained},  MSSC, MLRR , KMSSC, KMLRR \citep{abavisani2018multimodal} and DMSC \citep{abavisani2018deep} 
		The results of the baseline methods are from \citep{abavisani2018deep}.
	}
	\end{minipage}
\end{figure*}

\textbf{Handwritten Digits:}
We chose two famous handwritten digits datasets MNIST \citep{MNIST1998} and USPS \citep{hull1994usps} that  consist of grayscale digit images of size $28 \times 28$ pixels and $16 \times 16$ pixels, respectively.
To make multi-modal dataset,  each digit image in the MNIST dataset is paired with an arbitrary sample of the same digit identity but from USPS dataset.
This process guarantees that the images of both modalities are just sharing the same identity (label)  of the digit but not the style of the handwriting.
The handwritten digits datasets are used for single-modal training and also for multi-modal, with $M=2$, subspace clustering.

\textbf{Multi-modal Facial Components:}
We also evaluate the proposed method on the multi-modal facial dataset used in \citep{abavisani2018deep},
where the Extended Yale-B dataset \citep{lee2005YaleB} was used as the base and 4 facial components are extracted, by cropping eyes, nose and mouth, and formed 5 different modalities, including the whole face image.
All modalities are resized to images of size $28 \times 28$ pixels.
This dataset is composed of 64 frontal images of 38 individuals under different illuminations and is a standard dataset in subspace clustering studies.
For this multi-modal data, we train the general deep probabilistic multi-view model that extends the deep probabilistic CCA to arbitrary number of views \eq{eqn:gMBFA} (Appendix \ref{apdx:GPMV}).

\paragraph{Experimental design:}

To make a fair comparison, in this set of experiments, the encoders and decoders  are built of neural networks with similar architectures as those  used in \citep{abavisani2018deep},
except that our model does not require the self-expressive layer,
a linear fully connected layer with parameter matrix of size  $N\times N$ coefficients.
This key advantages of the proposed model, indeed, significantly reduces the total number of parameters specially for large input sizes.

Accordingly, the encoders (inference networks) of all modalities are composed of convolutional NN (CNN) layers while the decoders (observation networks) are built of transposed convolution layers.
\texttt{ReLU} gate is used as nonlinearity for all the  hidden units of the deep networks.
The encoders specify $ (\vmu\mi, \text{diag}(\evsigma\mi {i}^2)) = f\mi (\rvx\mi; \theta\mi )$,
where the variances are model by \texttt{softplus} function.
An extra encoder network models the canonical correlations, $\text{diag}(p_i) $, using the
\texttt{sigmoid} function as the output gate.
The observation likelihood functions of all the views, $p_{\theta\mi }(\rvx\mi \given \rvz\mi )$, are modeled by independent Bernoulli distributions
with the mean parameter being specified by decoder networks, $g\mi (\rvz\mi; \theta\mi )$;
the \texttt{sigmoid} functions are applied to estimate valid means of the distributions.
To train the parameters of deep generative model, we used ADAM optimization \citep{kingma2014adam} with learning rate of $.0002$ and default hyper-parameters and minibatch size of 200 data points.
Details of model architecture and experimental setup together with more empirical results are presented in appendix \ref{apdx:modelArch}.

Clustering is performed on the shared latent factor $\statevec$ using spectral clustering algorithm \citep{vonLuxburg2007tutorialSpectralClustering} on the \textit{k}-nearest-neighbor graph, with number of neighbors $k=5$.
As the last step, the spectral clustering discretizes  the real-valued  representation in the embedding space to extract the final partitioning.
Clustering performance are measured using clustering Accuracy rate (ACC), Normalized Mutual Information (NMI) \citep{cai2005documentclustering} and Adjusted Rand Index (ARI) \citep{rand1971metricARI} as performance metrics.

The clustering performance of the proposed method is compared against the well established subspace clustering methods
TM-MSC \citep{zhang2015TMMSC}, CMVFC \citep{cao2015constrained},  MSSC, MLRR , KMSSC, KMLRR \citep{abavisani2018multimodal} and DMSC \citep{abavisani2018deep} that are used as the baselines methods for multi-modal setting.
The  results summarized in Table \ref{tbl:clustering} show that 
the proposed deep generative model sets new state-of-the-art
which, subsequently, highlights that the proposed method can efficiently leverage the extra modalities and extract the common
underlying information  among the modalities, that is the cluster memberships.

\section{Conclusion}
In this work, we developed a simple, yet powerful, tool for multi-view learning based on the probabilistic interpretation of CCA.
A deep  generative probabilistic model for multi-view data was studied. 
It has been shown that following the theoretical formulation of the linear probabilistic CCA model in conjunction with variational inference principles for deep generative networks,
we can obtain a scalable end-to-end learning algorithm for  multi-view data.
Experimental results have shown that this can efficiently integrates the relationship between multiple views to obtain a more powerful representation which achieved 
state-of-the-arts performance on several downstream tasks.
These indeed suggest that the proposed method is a proper way of extending variational inference to deep probabilistic multi-view learning.

\bibliography{paper.bib}
\bibliographystyle{plainnat} 

\clearpage
\begin{appendices}
\section{Proof of theorem \ref{thm:PCCA}} \label{apdx:proof_PCCA}
The marginal mean and covariance
matrix of the joint views $ \rvz = (\rvz_1, \rvz_2)$ under the linear probabilistic model \eq{eqn:gPCCA} are 
$\vmu = \begin{pmatrix} \mW_1 \vmu_0 + \vmu_{\eps_1} \\ \mW_2 \vmu_0 + \vmu_{\eps_2} \end{pmatrix}$
and
$\mSigma = \mW \mW^{\top} + \mPsi$ 
where we define 
$\mW = \begin{pmatrix} \mW_1 \\ \mW_2 \end{pmatrix}$ and
$ \mPsi = \begin{pmatrix} \mPsi_1 & \zerovec \\ \zerovec & \mPsi_2 \end{pmatrix}$, therefore, similar to the proof in \citep{bachPCCA},
the negative log-likelihood of the data can be written as 
\begin{align*}
    \ell_1 & = \frac{n(d_1+d_2)}{2} \log 2\pi + \frac{n}{2} \log | \mSigma | \\
    & \quad \quad + \frac{1}{2} \sum_{j=1}^{n} \tr \mSigma^\inv (\rvz_j - \vmu)(\rvz_j - \vmu)^\top \\
    & = \frac{n(d_1+d_2)}{2} \log 2\pi + \frac{n}{2} \log | \mSigma | \\
    & \quad \quad + \frac{n}{2} \tr \mSigma^\inv \Tilde{\mSigma}
    + \frac{n}{2} (\Tilde{\vmu} - \vmu) \mSigma^\inv (\Tilde{\vmu} - \vmu)^\top
\end{align*}
Maximizing $\ell_1$ with respect to $\vmu$ results in a maximum
$\vmu = \begin{pmatrix} \mW_1 \vmu_0 + \vmu_{\eps_1} \\ \mW_2 \vmu_0 + \vmu_{\eps_2} \end{pmatrix} = 
\begin{pmatrix} \vmu_1 \\ \vmu_2 \end{pmatrix}$ and the negative log-likelihood is reduced to 
\begin{align*}
    \ell_1 & = \frac{n(d_1+d_2)}{2} \log 2\pi + \frac{n}{2} \log | \mSigma | + \frac{n}{2} \tr \mSigma^\inv \Tilde{\mSigma}
\end{align*}
The rest of the proof follows immediately along the line of proof in \cite{bachPCCA}. 

\section{Generalization of probabilistic CCA} \label{apdx:GPMV}
As an extension for probabilistic CCA to arbitrary number of views, \citep{archambeau2009sparse} proposed a general probabilistic model as follows:
\begin{align} \label{eqn:gMBFA2}
\rvz\mi =& \mW\mi \statevec_0 + \mV\mi \statevec\mi + \vmu\mi + \nu\mi, \forall m \in \{1, ..., M\}\\
& \quad \mW\mi \in \RR^{d\mi \times d_0},  \mV\mi \in \RR^{d\mi \times q\mi }, \nonumber
\end{align}
This model  can also be viewed as multibattery factor analysis (MBFA) \citep{klami2014GroupFactorAnalysis}, \citep{browne1980MBFA}
which describes the dependency between all the views by a single latent vector, $\statevec_0$, and explains away the view-specific variations by factors private to each view.
Moreover, limiting to single view, this model also includes the probabilistic factor analysis and probabilistic PCA as a special cases if  the view-specific dimensions are independent and isotropic \citep{archambeau2009sparse}.

Explaining  the view-specific variations  by the variance matrices for each view, we can represent the probabilistic multi-view model as
\begin{align} \label{eqn:PMV}
\statevec  &\sim \normal(\vmu_0, \eye_{d_0}), \nonumber\\
\rvz \mi \given  \statevec & \sim \normal(\mW\mi \statevec + \vmu_{\eps\mi }, \mPsi\mi ), \\
 \quad \mW\mi & \in \RR^{d \mi \times d_0}, \mPsi \mi \succcurlyeq 0, ~ m \in \{1, ..., M\} \nonumber
\end{align}
Where the latent factor $\statevec$ captures the common variations  between all the views hence describing the essence of multi-view data.
On the other hand, since the cluster memberships can be considered as a common information in all the views, 
this shared underlying representation is well suited for subspace clustering in multi-view setting.
Let $(\vmu\mi, \mSigma_{mm})$ be the mean and covariance matrices of $\rvz\mi $.
Inspired by the maximum likelihood solutions of probabilistic CCA in theorem \ref{thm:PCCA} and \citep{bachPCCA}, that reformulate the parameter estimate for the probabilistic model based on the classical CCA solutions, we can propose the following system of equation for the parameters of the probabilistic multi-view model
\begin{align} \label{eqn:param_PMV}
{\mW}\mi  &= \mSigma_{mm}^{1/2} \mU\mi \mP_{d_0}^{1/2} \mR \nonumber \\
{\mPsi}\mi  &= \mSigma_{mm} - {\mW}\mi {\mW}\mi^{\top} \nonumber \\
{\vmu}_{\eps \mi } &= \vmu\mi - {\mW}\mi \vmu_0
\end{align}
Where  $\mP_{d_0}=\text{diag}([p_0, ..., p_{d_0}])$ with diagonal entries $p_{j}\in [0., 1.]$
and $\mU\mi $ are composed of orthonormal vectors $\{\vu_{mi}\}$.
To simplify the model, we assume $\mU\mi = \mU, \forall m \in \{1,...,M\}$.
The equations in \eq{eqn:param_PMV} reduces to maximum likelihood estimate of PCCA for $m=2$ views, hence can be viewed as an extension of PCCA for multi-view with more than two views.
Defining the correlation matrix as
$\mC_{lm} := \mSigma_{ll}^{-1/2} \mSigma_{lm} \mSigma_{mm}^{-1/2}$ ,
equations in \eq{eqn:param_PMV} imply that $\mP_{d_0}$ and $ \mU$ are formed by the singular value and singular vectors of the correlation matrix, respectively, \ie $\mC_{lm} =  \mU \mP_{d_0} \mU^{\top} $,
therefore, analogous to the ML solution of PCCA,
$\mP_{d_0}$ and $ \mSigma_{mm}^{-1/2} \mU\mi $ can be interpreted as matrices of \textit{ canonical correlations} and \textit{canonical directions}.
This also implies that all the pairs of the views have similar correlation matrix.

In section \ref{sec:var_inf}, a closed form solution to infer $\{\vmu_{0}, \vmu_{\eps \mi }\}_{m=1}^M$ based on other variational statistics of the model is presented. 
We will also provide a simple treatments for $\mU$ and $\mR$.
As a consequence, given the first and second order moments of the views together with the diagonal matrix of canonical correlations $\mP_{d_0}$, one can infer the rest of the parameters for the multi-view generative model in \eq{eqn:PMV}.
This, in fact, simplifies the variational inference network to just learn a compact set of parameters.

It is worth noting that, although the parametrization of model is based on a single shared latent factor (and also a single correlation matrix to explain the relationship between all the views) but it can be seen that the contribution of the shared factor in $m^{th}$ view is controlled by the factor loading $\mW\mi $ that is, in turn, a function of
$\mP_{d_0}$ and $\mSigma_{mm}$, which is a view specific parameter.
Therefore, the shared factor does not equally influence the views but instead its effect on each view varies by the strength of its projection $\mW\mi \statevec$ which results in dissimilar cross-covariances for the different pairs of views, $\mSigma_{ml}, m \ne l $
therefore offering flexibility to model uneven dependencies between the views.


\section{Some proofs of section   \ref{sec:var_inf}} \label{apdx:proof_mean_PCCA}

\subsection{Proof of additive property of KL \eq{eqn:kl}}

Conditional independence of the latent variables
$\{\rvz\mi \given \statevec \}$ induced by the probabilistic graphical model of latent linear layer \eq{eqn:gPCCA} implies that the approximate posterior of the set of latent variables can be factorized  as
\begin{align}
     q_{\eta}(\rvz \given \rvx) = q_{\eta}(\statevec \given  \rvx)
    \prod_{m=1}^{M} q_{\eta}(\rvz\mi \given \statevec, \rvx).
\end{align}
In addition, assuming independent prior distribution on the latent variables, \ie 
$ p(\rvz) = p(\statevec) \prod_{m=1}^{M} p(\rvz\mi)$
leads to  

\begin{align*} 
&\KL [ q_{\eta}(\rvz \given \rvx) \Vert p(\rvz) ] = 
\int q_{\eta}(\statevec \given  \rvx)\prod_{m=1}^{M} q_{\eta}(\rvz\mi \given \statevec, \rvx) \times \\
& \quad \quad \quad  \log \frac{q_{\eta}(\statevec \given  \rvx)
    \prod_{m=1}^{M} q_{\eta}(\rvz\mi \given \statevec, \rvx)}
    {p(\statevec) \prod_{m=1}^{M} p(\rvz\mi)}  \\
& \quad \quad \quad  =  \int q_{\eta}(\statevec \given  \rvx) \log \frac{q_{\eta}(\statevec \given  \rvx)}{p(\statevec )} \\
& \quad \quad \quad   + \sum_{m=1}^{M} 
\int q_{\eta}(\rvz\mi \given \statevec, \rvx) \log \frac{q_{\eta}(\rvz\mi \given \statevec, \rvx)}{p(\rvz\mi)} \\
& = \KL [ q_{\eta}(\statevec \given \rvx) \Vert p(\statevec) ] +
\sum_{m=1}^{M}\KL [ q_{\eta}(\rveps\mi \given \rvx) \Vert p(\rveps\mi ) ] \hfill \blacksquare
\end{align*}

\subsection{Proof of Lemma \ref{thm:mean_PCCA}}
Assuming isotropic multivariate Gaussian priors on the latent variables as 
$\statevec \sim \normal(\mathbf{0}, \lambda_{0}^{-1} \eye)$,
$\rveps\m \sim \normal(\mathbf{0}, \lambda\m^{-1} \eye)$
and specifying the approximate posteriors as Gaussian distributed vectors with diagonal covariances 
results in closed form solutions for the KL divergence terms \citep{kingma2013VAE} as
\begin{align*}
	\KL [ q_{\eta}(\statevec \given \rvx) \Vert p(\statevec) ] &=
	\frac{1}{2}  \lambda_0 \| \vmu_{0} \|^2 + \nonumber \\
	&\frac{1}{2} \sum_{i = 1}^{d_0} (\lambda_0 - \log \lambda_0 -1) \nonumber \\
	\KL [ q_{\eta}(\rveps\m \given \rvx) \Vert p(\rveps\m) ] &=
	\frac{1}{2}  \lambda\m \| \vmu_{\eps\m} \|^2 + \nonumber \\
	&\frac{1}{2} \sum_{i = 1}^{d\m} (\lambda\m \evsigma_{mi}^2 - \log \lambda\m \evsigma_{mi}^2 -1) \nonumber 
\end{align*} 
Splitting the terms in the KL divergence to those depending on the mean variables and the remaining ones results 
\begin{align*} 
    \min_{\vmu_{0}} & \frac{1}{2}  \lambda_0 \| \vmu_{0} \|^2 +
    \frac{1}{2} \sum_{m=1}^{M} \lambda\m \| \vmu_{\eps\m} \|^2 + \mathcal{K}  \\
    \text{s.t. }
    \vmu_{\eps\m} &= \vmu\m - \mW\m \vmu_0, \quad \forall m \in \{1,..., M\}  \nonumber
\end{align*}
Now, solving this constraint optimization problem using the method of Lagrange multipliers leads to the optimal minimizer
 \begin{align*} 
 \vmu_{0}^* = (\lambda_0 \eye + \sum_{m=1}^{M} \lambda\m \mW \m^{\top} \mW \m)^{-1} (\sum_{m=1}^{M} \lambda\m \mW\m^{\top} \vmu \m).
 \end{align*}
This provides an analytical approach to optimally recover  $\vmu_{0}$  from the parameters of the model. 

\section{2-D embedding} \label{apdx:tsne}
Figures \ref{fig:TSNE_z1} 
depict the 2D t-SNE embeddings of the shared latent representations for multi-view and multi-modal setting, respectively.
They verify that the representation of the images of different classes are well separated in the shared latent space.\\

\begin{figure}[t!]\centering 
	\newcommand{\figpath}{./FigureTable/}
		{\label{fig:TSNE_phi_mean_v1} 
			\includegraphics[width=0.32\textwidth]{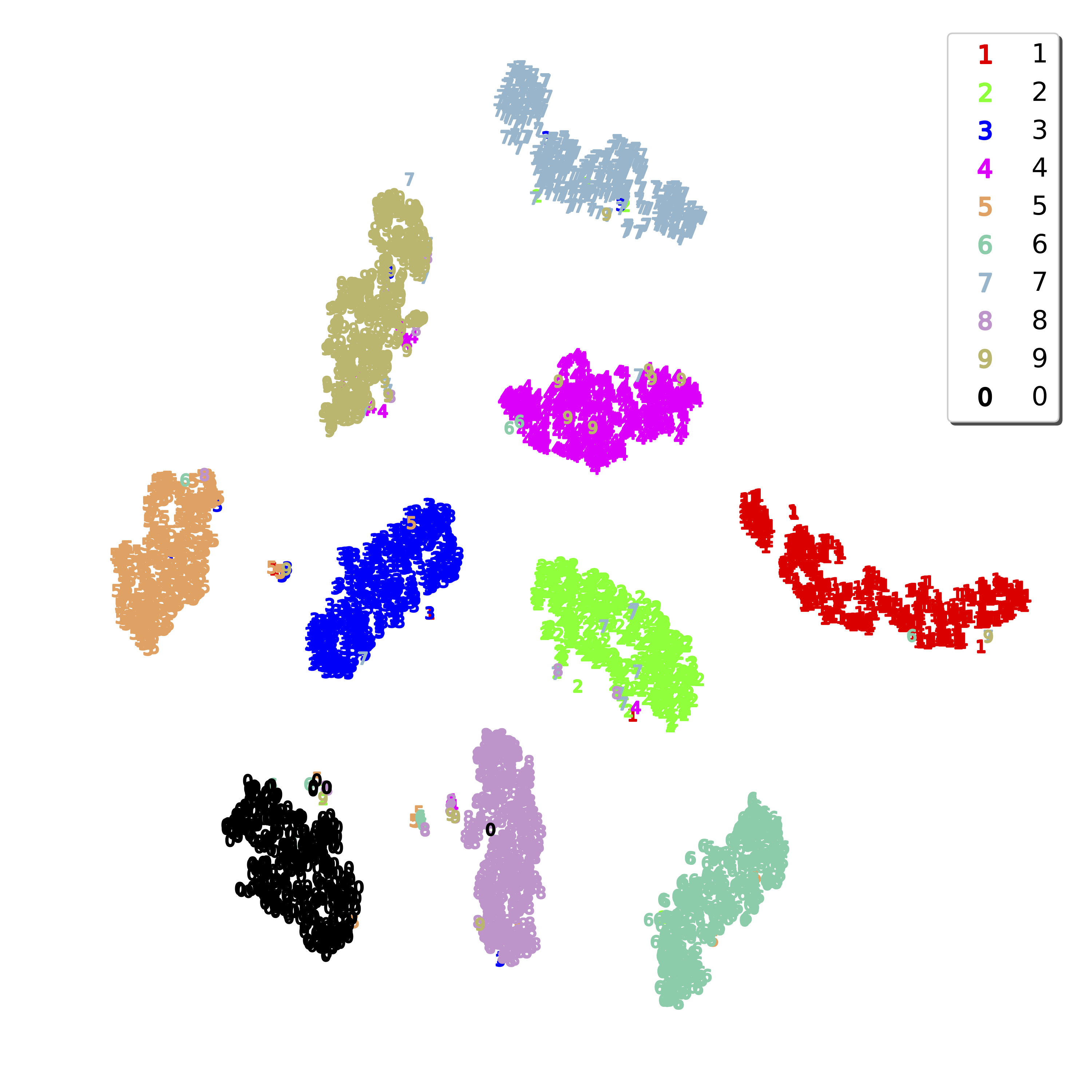}}
	\caption{\footnotesize \label{fig:TSNE_z1}
	2D t-SNE embedding of  the mean of shared representation $\vmu_0 $.
	} 
\end{figure}

\section{Model architecture and training procedure} \label{apdx:modelArch}
\subsection{Two-view noisy MNIST experiments} \label{apdx:modelArch_mnist}
The parameters of each algorithm are tuned through cross validation with grid search over 
$ p_{dropout} \in \{ .0, .2\}$
the variance of the shared representation $\statevec$
$ \lambda_0^\inv \in \{ 100., 500., 2000., 5000. \}$
and equal variance for residual errors $\rveps_1, \rveps_2$ in range $ \lambda_0^\inv \in \{ 8., 4., 2., 1., .5, .25, .125\} $.
Results are averaged over 3 trails.

The dimensionality of the shared representation was $d_0 = 30$ and 
the dimensionality of the latent factors were $d_1 = d_2 = 60$.

Weight decay of 0.0001 was applied as the regularization for all the parameters of NNs.

\subsection{Multi-modal clustering experiments}
\textbf{digits:}
The dimensionality of the shared representation was $d_0 = 30$ and 
the dimensionality of the latent factors were $d_1 = d_2 = 60$.

The parameters of each algorithm are tuned through cross validation with grid search over 
$ p_{dropout} \in \{ .0, .2\}$
the variance of the shared representation $\statevec$
$ \lambda_0^\inv \in \{ 1., 5., 20., 100., 500., 2000. \}$
and equal variance for residual errors $\rveps_1, \rveps_2$ in range $ \lambda_0^\inv \in \{ 8., 4., 2., 1., .5, .25, .125\} $.

\textbf{Yale-B facial components:}
The dimensionality of the shared representation was $d_0 = 120$ and 
the dimensionality of the latent factors were $d_1 = d_2 = 160$.

The parameters of each algorithm are tuned through cross validation with grid search over 
$ p_{dropout} \in \{ .0, .2\}$
the variance of the shared representation $\statevec$
$ \lambda_0^\inv \in \{.2, 1., 5., 20., 100., 500., 2000. \}$
and equal variance for residual errors $\rveps_1, \rveps_2$ in range $ \lambda_0^\inv \in \{ 8., 4., 2., 1., .5, .25, .125\} $.

In both experiments a weight decay of 0.0001 was applied as the regularization for all the parameters of NNs.

\end{appendices}

\end{document}